\documentclass{article}

\usepackage{latexsym,amsmath, amssymb, amsthm}
\usepackage{fullpage, hyperref} 
\usepackage{graphics,subfig}
\usepackage{algorithm,algorithmic}
\algsetup{linenosize=\small}
\usepackage{xfrac,multirow}
\usepackage{tikz}

\usepackage{epstopdf}
\usepackage{float}

\newtheorem{lemma}{Lemma}

\newtheorem{defn}{Definition}
\newtheorem{remark}{Remark}
\newtheorem{theorem}{Theorem}

\newcommand{\mean}{{\rm mid}}

\title{Loss of Distributed Coverage Using Lazy Agents Operating Under Discrete, Local, Event-Triggered Communication}
\author{Edward Vickery\footnote{EV is Head of Solution Operations at SHL, London, United Kingdom. Email: ted.vickery@gmail.com} \ and 
Aditya A. Paranjape\footnote{AP is Senior Scientist at TCS Research, Tata Consultancy Services Ltd., Pune. He is also 
Honorary Lecturer at Imperial College London and Visiting Associate Professor at the Indian Institute of Technology 
Bombay. Email: aditya.paranjape@gmail.com.}}
\date{}

\begin{document}
\maketitle

\section*{Abstract}
Continuous surveillance of a spatial region using distributed robots and sensors is a well-studied application in the area of multi-agent systems. This paper investigates a practically-relevant scenario where robotic sensors are introduced 
asynchronously and inter-robot communication is discrete, event-driven, local and asynchronous. Furthermore, we work with lazy robots; i.e., the robots seek to minimize their area of responsibility by equipartitioning the domain to be covered. 
We adapt a well-known algorithm which is practicable and 
known to generally work well for coverage problems. For a specially chosen geometry of the spatial domain, we show that 
there exists a non-trivial sequence of inter-robot communication events which leads to an instantaneous loss of coverage when the number of robots exceeds a certain threshold. The same sequence of events preserves coverage and, further, leads to an equipartition of the domain when the number of robots is smaller than the threshold.
This result demonstrates that coverage guarantees for a given algorithm might be sensitive to the number of robots and, therefore, may not scale in obvious ways. It also suggests that when such algorithms are to be verified and validated prior to field deployment, the number of robots or sensors used in test scenarios should match that deployed on the field.

\section{Introduction}
The development of autonomous vehicles in recent years has expanded the range of tasks that can be carried out without human intervention. This paper explores one such complex task, namely that of continuous surveillance of an environment using autonomous mobile robots. The problem of surveillance and monitoring has multiple facets, depending on the nature of the mission and the sensors 
and robots involved in the task \cite{chung18}. The problem 
addressed in the paper is effectively that of partitioning
an environment for the purpose of 
continuous surveillance using a distributed scheme
which relies on local, event-triggered communication between mobile robots. The robots seek to minimize their individual areas of
coverage while ensuring that the environment as a whole is covered.
The distributed nature of the communication and task allocation 
between the robots leads to the following question with
practical ramifications: {\it if a reasonably designed coverage algorithm is proven to work for some non-trivial range of numbers of robots, can it fail to work when the number of robots is changed to outside the proven range?} We answer this question in the affirmative by constructing an example.

\subsection{Overview of the literature}\label{sec:lit}
Optimal sensor placement problems are built upon the premise 
that the number of sensors and their placement can be mapped to an objective function which needs to be either maximized or minimized 
subject to constraints related to the sensor and the environment.
When the objective function satisfies submodularity properties, it 
is possible to use greedy algorithms to solve these problems up to 
provable bounds \cite{clark17}. This approach has been investigated for 
sensor placement in \cite{krause2006near, krause2007near, near_optimal_sensor}. Similar work in \cite{singh2006efficient} presents an efficient path planning 
algorithm for multiple agents in the presence of resource constraints for the
agents. Integer programming techniques can be employed if the structure of the problem permits appropriate spatial discretization \cite{char22}. 

All of these approaches assume the existence of a centralized decision-making system that is aware of all active sensors and can allocate them to individual positions. The centralized decision-making approach becomes impractical when robotic agents are inserted into environments where continuous communication with a central node is impossible (e.g., hostile environments or ones which are naturally packed with communication obstacles). In such cases, inter-robot communication is asynchronous, local and event-driven. A canonical coverage problem is that of partitioning the environment between the mobile robots equitably and dynamically using a
distributed scheme. 

Distributed coverage algorithms that achieve a Voronoi partition of convex environments using a distributed approach were proposed in \cite{Coverage}. The distributed scheme is dynamic, in that the
Voronoi partitions are scaled and rearranged dynamically to yield an optimal partition together with a motion planning algorithm for the mobile robots. This approach has been extended to non-convex domains \cite{Heterogeneous, Nonconvex, bhat13}, to robots with finite communication radii \cite{r2}, to environments with unknown sensory functions \cite{schwager09} and to problems where the cost function for
a sensor can be augmented or mixed with that of its neighbors \cite{schwager11}. A similar approach can be used to balance raw
coverage (i.e., the area covered) with the quality of the coverage for a spatial distribution of events \cite{an21}. 

The distinction between continuous coverage and reliable detection
of flag events (for which continuous coverage is sufficient but
not necessary) is brought out for time varying environments in \cite{dynamic_coverage} wherein sensor movement is optimized in 
order to maximise the probability of identifying flag events.

A continuous flow of information may not be required in order to maintain coverage. An algorithm which uses events triggered by individual agents in order to guarantee coverage is presented in \cite{Selftriggered}.
An algorithm which caters to gossip-based inter-robot communication has been investigated in 
\cite{Gossip} where random pairs of agents are allowed 
to communicate and relocate based on local information exchange. 

\subsection{Contribution}
The algorithm analyzed in this paper is an adaptation of the algorithm in \cite{Gossip} for robots that use a lazy scheme (in a sense 
which will be made precise later) to repartition or resize their areas of responsibility. We wish to examine whether such lazy behavior (which may be viewed as the equivalent of greedy behavior in 
optimization problems with an agent-level penalty function that dominates the global reward function for coverage) can result in a
loss of coverage and the conditions under which coverage is lost.

Towards that end, we construct a simplified example and a sequence of events which leads to an instantaneous loss of coverage when the number of robots exceeds a non-trivial threshold. Interestingly, the same sequence of events actually leads to an equipartition of the domain 
(i.e., the optimum solution) for a smaller number of robots. This demonstration suggests that the success of multi-agent algorithms operating in the
presence of restricted communication might be sensitive to the number of 
agents involved, above and beyond the known complexities that arise
due to the ``scale'' of the problem or the geometry of the environment.

The rest of the paper is organized as follows. Preliminaries are laid out in Sec.~\ref{sec:prelim}, including the class coverage algorithms considered in the paper. The main theoretical results of the paper are presented in Sec.~\ref{sec:main}, and numerical experiments are
used to generalize those results in Sec.~\ref{sec:sim}.


\section{Preliminaries}\label{sec:prelim}
\subsection{Unit circle}
Let $S^1$ denote the unit circle parametrized by the angular variable $\theta \in [0,\,2\pi)$. We write 
$\theta \in S^1$. Let $\theta_1,\,\theta_2 \in S^1$. The length of the shorter arc between these points is given by
\begin{equation}
d(\theta_1,\,\theta_2) = d(\theta_2,\theta_1) = \begin{cases} |\theta_2 - \theta_1|, & |\theta_2 - \theta_1| < \pi \\ 
2\pi - |\theta_2 - \theta_1| & {\rm otherwise}\end{cases}
\label{eq:arcl}
\end{equation}
The centroid of the two points along the short arc is given by
\begin{equation}
\mean(\theta_1,\theta_2) = \begin{cases} 0.5(\theta_2 + \theta_1), & |\theta_2 - \theta_1| < \pi \\ 
{\rm mod}(\pi + 0.5(\theta_2 + \theta_1),~2\pi)& {\rm otherwise}\end{cases}
\label{eq:cg}
\end{equation}

The set $S^1$ is isomorphic to a unit circle in the complex
plane via the invertible mapping $T(\theta) = e^{j\theta},~\theta \in S^1$.
\begin{defn}[Positive or clockwise rotation]\label{def:clockwise} An arc in $S^1$ 
is said to be a clockwise or positively directed
arc from $\theta_l \in S^1$ to $\theta_u \in S^1$ 
if there exists a continuous function $c: [0,1]\to\mathbb{R}$
and $\omega \geq 0$ such that: (i) every point on the arc can be represented as $c(t)$ for some $t\in[0,1]$, with $c(0) = \theta_l$ and $c(1) = \theta_u$; (ii) $T(c(t)) = e^{j\omega t} T(c(0))$; and (iii) $c(t_1) = c(t_2)$ for $t_1 \neq t_2$ iff $\theta_l = \theta_u$ and $\omega = 0$.
\end{defn}

\begin{defn}
We say that $\theta_u \succ \theta_l$ if the shortest
arc from $\theta_l$ to $\theta_u$ (in that order) is traced via a
clockwise rotation.
\end{defn}

\begin{defn}\label{def:oplus}
We define the addition operator $\oplus$ to denote a clockwise rotation on $S^1$; i.e., $\theta_1 \oplus \theta_2$ denotes a 
clockwise rotation of magnitude $\theta_2$ starting from $\theta_1$.
It is clear that $\theta_1 \oplus \theta_2 = \theta_2 \oplus \theta_1$. We define the operator $\ominus$ to denote an 
anticlockwise rotation on $S^1$: $\theta_1 \ominus \theta_2$ denotes an anti-clockwise rotation of magnitude $\theta_2$ starting with $\theta_1$. Finally, the standard summation operator $\sum(\cdot)$ will denote a sum using $\oplus$ when the arguments belong to $S^1$.
\end{defn}

If $\theta_2 \succ \theta_1$, Definition~\ref{def:oplus} allows us to
write
\begin{equation}
\theta_1 = \mean(\theta_1,\,\theta_2) \ominus 0.5\,d(\theta_1,\,\theta_2),~~
\theta_{2} = \mean(\theta_1,\,\theta_2) \oplus 0.5\,d(\theta_1,\,\theta_2), 
\label{eq:addsubt}
\end{equation}

\subsection{Partitions of closed, bounded regions in $\mathbb{R}^2$}
Let $Q \subset \mathbb{R}^2$ be a closed, bounded domain containing $N \geq 1$ agents or sensors. Let $p_i \in Q$ denote the position of the $i^{\rm th}$ agent.
A partition of size $N$ of $Q$ is a set $V = \{V_1,\dots,\,V_N\}$, where $V_i \subset Q$ are closed and satisfy $\cup_{i=1}^N V_i = Q$. We may further prescribe that each agent lie inside its own partition. The coverage problem usually considered in the literature involves finding a partition $V^\ast$ which solves the problem
\begin{equation}\label{eq:cost}
V^\ast = \arg\min_V \mathcal{H}_V(\mathcal{P}) \triangleq \sum_{i=1}^{n}\int_{V_i} f_i(\|q-p_i\|)d\phi(q)
\end{equation}
where $\phi: Q\to \mathbb{R}_{+}$ satisfying $\int_{Q}d\phi(q) = 1$ is a weighting function 
(also called the sensing function) and $f_i:\mathbb{R}_+ \rightarrow\mathbb{R}_+$ represents the sensing performance of the 
agent $i$ as a function of its distance from the sensed location $q\in Q$. Lloyd's algorithm and its variants \cite{Controlbook}
are used to partition $Q$ dynamically into Voronoi cells which leads, in turn, to the optimal partition which is itself
a Voronoi partition. Partitions may also be constructed organically to satisfy sensing constraints, 
such as in \cite{Artgalleries}, without solving the optimization problem above. In this paper, we will consider 
equipartitions which are defined as follows.
\begin{defn}\label{def:equip}
For $N$ agents located at $\{p_1,\dots,\,p_N\}$, $p_i \in Q$ for all $i$, we call $W = \{W_1,\dots,\,W_N\}$ an {\em equipartition} of $Q$ if it satisfies: (i) $W_i \subseteq Q$ is closed for all $i$; (ii) $p_i \in W_i$; (iii) ${\rm area}(W_i) = {\rm area}(W_j) = {\rm area}(Q)/N$ for all $i,\,j$; and (iv) $\cup_{i=1}^N W_i = Q$. It follows that ${\rm int}(W_i) \cap {\rm int}(W_j) = \emptyset$ for $i \neq j$ and 
$p_i \in {\rm int}(W_j)$ if and only if $i=j$.
\end{defn}
It is clear that an equipartition $W$ solves \eqref{eq:cost} when $\phi(q)$ is uniform and the function $f_i(\cdot)$ is identical for all $i$.

\subsection{Gossip-based coverage algorithms}\label{sec:cover}
In this section, we describe the {\em class of} distributed coverage algorithms, based on \cite{Gossip}, which are subsequently specialized and analyzed in the paper. Briefly, a sufficient condition for coverage is that 
the spatial domain equals the union of the areas assigned to individual agents.
This assignment is carried out in a distributed manner by the agents with
the aim of creating an equipartition while ensuring that coverage is not lost. 
We make four important assumptions: (1) the agents are assumed to be {\em lazy} in a sense that will be made precise presently; (2) every agent is aware of the geometry of the environment before entry but not of the other agents; (3) at most one {\em interaction event} (defined presently)
can occur at any given point in time, and (4) agents' response to events, including repartitioning, is instantaneous and is thus complete before the next event. Unlike \cite{Selftriggered}, we assume that an interaction event cannot be triggered by one or more agents. Rather, we model it as a random occurrence, which is a proxy for two agents
coming within their mutual communication radius either in the course
of exploring their area of responsibility or responding to an 
environmental event.

\begin{defn}
Consider an agent in a closed, connected domain $Q$ with area
$|Q|$ and suppose that it has knowledge $K$ about the agents in $Q$,
with $|K|$ equal to the number of agents that it is aware of 
(including itself). Suppose that the agent allocates itself a domain $A \subseteq Q$ with area $|A|$. The agent is said to be lazy with an $\epsilon$ degree of 
altruism if 
$$
|A| = \frac{|Q|}{|K|} + \epsilon
$$
where $\epsilon$ can be time-varying and agent-specific. The agent 
is said to be lazy (with no mention of altruism) if $\epsilon = 0$.
\end{defn}

\begin{remark}
The numerical area $|A|$ of a partition $A$ can be calculated by scaling
using the sensor function $\phi(q)$ ($q \in Q$) to ensure that the 
partitions are equitable. We assume that $\phi(\cdot)$ is uniform, so that $|A|$ is the usual Euclidean area of $A$.
\end{remark}

\begin{defn}
An interaction event, identified by the time $t$ at which it occurs, is defined as an interaction between precisely two agents $i$ and $j$. The interaction consists of (i) updating each agent's knowledge 
$K_i,\,K_j \leftarrow K_i \cup K_j$, and (ii) repartitioning 
and resizing of their individual areas of responsibility as per
the guiding algorithm. Since we do not model environmental events
in this paper, we will use the word event to refer hereafter
to an interaction event.
\end{defn}

There are two types of (interaction) events. In the first type,
an agent enters the domain for the first time, with no
knowledge of the other agents, and encounters an existing agent. It 
makes a copy of the knowledge possessed by the existing agent 
and the two agents partition {\em the existing} agent's area equitably
into disjoint halves while also taking over responsibility for some 
of the surrounding area. 
In the second type of events, an agent that is already in the
domain interacts with another agent that is also in the domain. The two agents update each other's information so that both possess the union of their prior individual information. They repartition their individual areas as per the problem-specific guiding algorithm and the 
process continues. The pseudocode is described in Algorithm~\ref{alg:general}. 

\begin{algorithm}
\caption{General algorithm for partitioning a domain}
\label{alg:general}
\begin{algorithmic}
\REQUIRE A domain $Q \in \mathbb{R}^2$ and $N$ agents in all
\STATE Initialize: time $t = 0$; initial positions and areas
of responsibility for the agents already in $Q$. The union of the areas of responsibility of agents in $Q$ equals $Q$
\WHILE{Stopping condition not reached}
\IF{Enters new agent $j$}
\STATE New agent $j$ interacts with agent $i$ already in $Q$
\STATE Knowledge exchange $K_j \leftarrow K_i$
\STATE Disjoint areas of responsibility assigned $A_i,\,A_j\in Q$; $|A_i| = |A_j| \sim 1/|K_i| + \epsilon[t]$. 
\ELSE{}
\STATE Existing agents $i, j$ interact
\STATE Knowledge sharing: $K_{i}= K_{j} \leftarrow K_i \cup K_j$
\STATE Areas of responsibility reassigned so that 
$|A_{\{\cdot\}}| \sim Q/|K_{\{\cdot\}} + \epsilon_{\{\cdot\}}[t]$
\ENDIF
\STATE Stopping condition reached if $t = t_{\max}$ or $Q$ is equipartitioned
or no further events are feasible
\ENDWHILE
\end{algorithmic}
\end{algorithm}

We note that several steps of Algorithm~\ref{alg:general} have
been deliberately left vague. These can be made precise for individual problems, as we illustrate in the next section. We have only prescribed that (i) exchange of knowledge should correspond to all agents learning the union of their prior knowledge, and (ii) the areas assigned or reassigned to all participants in an interaction should have equal magnitudes up to the degree of altruism of individual agents.


\section{Main Results}\label{sec:main}
In this section, we consider a simplified coverage problem and
apply a corresponding manifestation of Algorithm~\ref{alg:general} to solve 
it. We assume that the agents are lazy (i.e., $\epsilon = 0$ uniformly for all agents). We construct a sequence of events
in Sec.~\ref{sec:main2} which leads to equipartition when the
number of agents is less than a certain threshold, 
and to loss of coverage when the number of agents exceeds the threshold.

\subsection{Geometry of the domain and some notation}
Let $Q_A = \{(x,y)\in\mathbb{R}^2~|~\delta^2< x^2 + y^2 \leq 1\}$ denote the domain of interest
in $\mathbb{R}^2$, where $0 < \delta \ll 1$. Clearly, $Q_A$ is a unit disc with a small hole at the centre. 
We can recast the problem into one of partitioning the domain $Q = S^1$ (the unit circle). It is evident 
from basic geometry that a partition of $S^1$ can be mapped to an equivalent angular slice of $Q_A$. Thus, 
an equipartition of $S^1$ can be mapped to an equipartition of $Q_A$.

The angular position of the $i^{\rm th}$ agent in $Q = S^1$
is denoted by $\theta_i \in [0,\,2\pi)$ once it is introduced in the domain, 
and $n_i \geq 1$ denotes the number of agents that it is aware of, including itself. The arc of dominance of the agent $i$ is denoted
by $[\theta_i^l,\,\theta_i^u]$ with $\theta_i^u \succ \theta_i^l$ and
$\theta_i^c = \mean(\theta_i^l,\,\theta_i^u)$. Finally, let $S_i = d(\theta_i^l,\,\theta_i^u)$. We note that all of these variables
are functions of time $t$; for brevity, we omit that argument unless necessary.
\begin{defn}\label{defn:lazy}
An agent is said to be {\em lazy} if its arc of dominance on
$S^1$ has length $S_i = 2\pi/n_i$.
\end{defn}
The following result can be derived readily using basic
geometry and we omit a formal proof.
\begin{lemma}\label{lem:sij}
Suppose two lazy agents $i$ and $j$ are aware of $n_i$ and $n_j $ agents, respectively. Then,
the overlap between their areas of dominance $S_{i,j} > 0$ if and only if 
$$
\frac{\pi}{n_i} + \frac{\pi}{n_j} - d(\theta_i^c,\,\theta_j^c) > 0
$$
Moreover, if $n_i > 1$ and $n_j > 1$, the overlap is given by
\begin{equation}
S_{i,j} = {\min}\left({\max}\left(\frac{\pi}{n_i} + \frac{\pi}{n_j} - d(\theta_i^c,\,\theta_j^c),~ 0\right),\,\frac{2\pi}{n_i},\,\frac{2\pi}{n_j}\right)
\end{equation}
If $n_i = 1$, $S_{i,j} = 2\pi/n_j$, and likewise if $n_j = 1$.
\end{lemma}

\begin{defn}\label{defn:cul}
For every agent $j$, we denote the area of overlap with its immediate clockwise and anti-clockwise neighbors by $C^u_j$ and $C^u_l$, respectively. In terms of the notation introduced
in Lemma~\ref{lem:sij},
\begin{eqnarray*}
C_j^u &=& S_{j,k},~k = \arg\min_{i\neq j}\{d(\theta^c_j, \theta^c_i)~|~\theta_i^c \succ \theta^c_j\} \\
C_j^l &=& S_{j,k},~k = \arg\min_{i\neq j}\{d(\theta^c_j, \theta^c_i)~|~\theta_i^c \prec \theta^c_j\}
\end{eqnarray*}
\end{defn}

\subsection{Addition of a new agent and repartitioning}
Suppose that a new agent $k$ is added to $Q$ at time $t$
and it interacts with an agent $j$ that is already in $Q$. 
As per Algorithm~\ref{alg:general}, they update their 
knowledge of the number of agents so that 
$n_k[t] = n_j[t] = n_j[t-1] + 1$. Thereafter, the two 
agents $j$ and $k$ assign themselves an area of responsibility of size $2\pi/n_k[t]$ such that the two arcs intersect at the point $\theta_j^c[t-1]$ (i.e., at the midpoint of agent $j$'s previous area of responsibility). The pseudocode for this process is
presented in Algorithm~\ref{alg:agent added}.

\begin{algorithm}[htb]
\caption{New agent repartition algorithm}
\label{alg:agent added}
\begin{algorithmic}
\REQUIRE agent $i$ in $Q$ located at $\theta_i^c$ and aware of $n_i(\geq 1)$ agents
\REQUIRE agent $j$ enters $Q$ and interacts with $i$
\STATE $n_j \leftarrow n_i + 1$, $n_i \leftarrow n_i+1$
\STATE $\theta_{i}^u \leftarrow \theta_i^c$, $\theta_{j}^l \leftarrow \theta_i^c$
\STATE $\theta_{i}^l \leftarrow \theta_i^u \ominus 2\pi/n_i$, $\theta_{i}^c \leftarrow \mean(\theta_{i}^l,\,\theta_{i}^u)$
\STATE $\theta_{j}^u \leftarrow \theta_{j}^l \oplus 2\pi/n_j$, $\theta_{j}^c \leftarrow \mean(\theta_{j}^l,\,\theta_j^u)$
\end{algorithmic}
\end{algorithm}

The next result shows that the process of adding a new agent 
is not detrimental to instantaneous coverage in itself.
\begin{lemma}\label{prop:1} Suppose that the domain $Q = S^1$ is covered by $k > 1$ agents and suppose that a new agent $k+1$ is added to the domain $Q$. Then, Algorithm \ref{alg:agent added} ensures that instantaneous coverage is not lost. Moreover, 
the resulting area of dominance of agent $k+1$ overlaps 
with that of at least one other agent in $[1,\,k]$
\end{lemma}
\begin{proof} When the $(k+1)^{\rm th}$ is introduced, let
$i$ denote the index of the agent with whom it interacts.
These two agents update their knowledge of the number of 
agents in $Q$ to $n_i+1$ as per Algorithm~\ref{alg:agent added}, where $n_i$ is the number of agents known to the agent $i$ before interacting with agent $k+1$. The two agents also assign themselves disjoint partitions of size $S_i = S_{k+1} = 2\pi/(n_i+1)$ each, which yields a joint coverage of size $4\pi/(n_i+1)$. We 
note that 
\begin{equation}
4\pi/(n_i+1) > 2\pi/n_i~{\rm for}~n_i > 1
\label{eq:temp0}
\end{equation}
Algorithm~\ref{alg:agent added} ensures that the 
areas of dominance of agents $i$ and $k+1$ are symmetric 
about the centroid of the previous area of dominance (labeled
temporarily as $S_i^{old}$) of agent $i$. From \eqref{eq:temp0}, it follows that the combined area of dominance of agents $i$ and $k+1$ is a superset of $S_i^{old}$. Since $Q$ was covered fully 
by the $k$ agents before the entry of agent $k+1$, agents other
than $i$ and $k+1$ cover the area outside $S_i^{old}$. This completes the proof. \end{proof}

Lemma~\ref{prop:1} implies that (i) the domain $Q$ is fully covered, with redundant coverage in some areas, at each step of agent addition, and (ii) the partition of $Q$ is not an equipartition at the end of agent addition. Since the agents would ideally aim for an equipartition, that would require further interaction between the agents.

\subsection{Sequential interaction of lazy agents}\label{sec:main2}
In this section, we construct the sequence of interactions between agents. The sequence
has two components. First, the agents are introduced
sequentially as described in Algorithm~\ref{alg:main}. 
The sequence continues with inter-agent interaction chosen 
from one of Algorithm~\ref{alg:repartition} and Algorithm~\ref{alg:repartition1}.

\subsubsection{Addition of agents}
In the first step (Algorithm~\ref{alg:main}), agents are
introduced sequentially with further prescription that agent $j > 1$ interacts only with agent $j-1$ upon entering the domain. 
This continues until the last agent $N$ is introduced.

\begin{algorithm}[htb]
\caption{Special case: sequential addition and interaction}
\label{alg:main}
\begin{algorithmic}
\STATE Initialize: domain $Q = S^1$; number of agents $n > 2$
\STATE Initialize: ${\rm agents\_added} = 2$; $t = 2$; $n_1 = n_2 = 2$
\STATE Initialize: $\theta_1^c = 0$, $\theta_2^c = \pi$
\WHILE{${\rm agents\_added} < n$}
\STATE Update time $t \leftarrow t + 1$
\STATE Introduce agent $t$; $agents\_added \leftarrow agents\_added + 1$
\STATE Agent $t$ interacts with agent $t-1$ per Algorithm~\ref{alg:agent added}
\STATE Update $n_{t-1}$, $n_t$, $\theta^c_{t-1}$, $\theta^c_t$, and $n_{t-1} = n_t = t$
\ENDWHILE
\end{algorithmic}
\end{algorithm}

We derive analytical expressions for $\theta_i^c[n]$; i.e., the centroid locations
after all $N$ agents have been introduced in the system.

\begin{lemma}\label{lem:thetaR1}
Suppose that $N > 2$ agents are added to the domain $Q = S^1$ as per Algorithm~\ref{alg:main}. Then, after all $N$ agents have entered $Q$ and pending any further interactions between agents, the positions of the $N$ agents are given by
\begin{eqnarray}
\nonumber & & \theta^c_1[N] = 0,~\theta_2[t] = \frac{2\pi}{3} \\
& & \theta^c_p[N] = \pi \oplus \sum_{m=3}^p \frac{\pi}{m} \ominus \frac{\pi}{p + 1},~p\in[3,\,N-1] \\
& & \theta^c_n[N] = \pi \oplus \sum_{m=3}^N \frac{\pi}{m}
\end{eqnarray}
\end{lemma}

\begin{proof}
We prove this result by considering a process wherein, 
at each time instant $t > 0$ , an agent is added to the domain $Q = S^1$. Clearly, at time $t = j$, 
the $j^{\rm th}$ agent gets added and it communicates only with agent $j-1$. The domain is partitioned 
as per Algorithm~\ref{alg:agent added} and the agents move to their individual centroids. This 
process continues until all agents $N$ are added to the system (i.e., until $t = N$).

At $t = 2$, when agent $2$ enters $Q$, agents $1$ and $2$ move to locations which are diametrically
apart; without loss of generality, we write $\theta_{1}^c[2] = 0$ and $\theta_{2}^c[2] = \pi$. 
Their mutual areas of coverage do not overlap. 

At time $t = 3$, agent $3$ enters $Q$ and communicates with agent $2$; agent $1$ does not move.
Agents $2$ and $3$ assign themselves domains of size $2\pi/3$ each; the common boundary of this domain is at the location of agent $2$'s centroid at time $t=1$. The new centroid locations are thus given by
\begin{eqnarray*}
\theta_2^c[3] = \pi - \frac{\pi}{3} = \frac{2\pi}{3},~~\theta_3^c[3] = \pi + \frac{\pi}{3} = \frac{4\pi}{3}
\end{eqnarray*}
This process can be repeated to yield, by induction,
that the centroid locations after adding the $j^{\rm th}$ agent at time $3 \leq j < n$ are given by
\begin{eqnarray}
\theta^c_k[j] &=& \pi \oplus \sum_{m=3}^k\frac{\pi}{m} \ominus \frac{\pi}{k+1},~~3\leq k \leq j-1 \nonumber \\
\theta_k^c[j] &=& \pi \oplus \sum_{m=3}^j\frac{\pi}{m},~k=j
\label{eq:centroid0}
\end{eqnarray}
Setting $j = N$ completes the proof. 
\end{proof}

\subsubsection{Sequential interaction between the agents 
in $Q$} Once all $N$ agents have been introduced as per 
Algorithm~\ref{alg:main}, we define a sequence of events labeled by $\mathcal{T} = \{k\}$, $k \in \mathbb{N}$ and $k \leq N-2$, such that at instant $k$, the agent $N-k$ interacts with agent $N-k-1$ if $S_{N-k, N-k-1} > 0$  or if $\theta^u_{N-k-1} = \theta^l_{N-k}$
or $\theta^l_{N-k-1} = \theta^u_{N-k}$ (i.e., the borders overlap at one end, which we denote compactly as $\theta^{u,l}_{N-k-1} = \theta^{l,u}_{N-k}$).
Else, the sequence is halted. This sequence is enumerated formally in Algorithm~\ref{alg:repartition}. 

\begin{algorithm}[htb]
\caption{Pairwise interaction algorithm - 1}
\label{alg:repartition}
\begin{algorithmic}
\REQUIRE $N$ agents introduced as per Algorithm~\ref{alg:main}
\STATE $k = 1$, $run\_flag = 1$
\WHILE {$run\_{flag} = 1$ and $k \leq N-2$}
\IF {$S_{N-k,N-k-1} > 0$ or $\theta^{u,l}_{N-k-1} = \theta^{l,u}_{N-k}$}
\STATE Update knowledge of $(N-k-1)^{\rm th}$ agent
\STATE $\theta_{N-k-1}^c = \theta_{N-k}^c \ominus 2\pi/N$
\ELSE
\STATE $run\_{flag} = 0$; halt sequence
\STATE Check for coverage
\ENDIF
\STATE $k \leftarrow k + 1$
\ENDWHILE
\end{algorithmic}
\end{algorithm}

\begin{remark}
\label{rem:1}
We recall that $\sum_m (\pi/m)$ is unbounded on $\mathbb{R}$.
$$
1 + \sum_{m=3}^p \frac{1}{m} \approx \begin{cases}
1.95 & p = 6 \\
2.09 & p = 7 \\
3.999 & p = 50 
\end{cases}
$$
Thus, on $S^1$, after Algorithm~\ref{alg:main}
is implemented, $j > k$ does not imply that $\theta_j > \theta_k$  (notice the use of 
$>$ rather than $\succ$). Informally, if the agents could be viewed as connected by a 
material thread, this thread could circle $S^1$ multiple times depending on $N$ (at least twice for $N > 50$).
\end{remark}

Remark~\ref{rem:1} portends significant complication of our calculations for large values of $N$. However, it is possible to show that coverage is lost even for a relatively small number of 
agents. The next theorem shows that the events in Algorithm~\ref{alg:repartition} do {\em not} lead to an instantaneous loss of coverage if the number of agents is
less than $7$. It also shows that coverage can be lost for $N=7$ agents. We generalise
this result later. 

\begin{remark}\label{rem:2}
If the number of agents is capped at $N \leq 7$, it is easy to show that $\theta^c_j[N] > \theta^c_k[N]$ if $j > k$ for $j,\,k \leq N-1$. Thus, as long as $\theta^c_{N-1}$ does
not shift, we can use the usual algebraic operators $+$ and $-$ in place of $\oplus$
and $\ominus$ for calculating the centroid locations resulting from Algorithm~\ref{alg:repartition}. 
\end{remark}

\begin{theorem}\label{thm:main1}
Let the number of agents be bounded by $N \leq 7$. 
Suppose that the agents are introduced following Algorithm~\ref{alg:main} and they then interact as per 
Algorithm~\ref{alg:repartition}. Then, Algorithm~\ref{alg:repartition} terminates with
loss of instantaneous coverage if and only if $N=7$. Moreover, for $N \leq 6$, 
Algorithm~\ref{alg:repartition} terminates with an equipartition of $Q$.
\end{theorem}
\begin{proof}
The case for $N = 1$ and $N=2$ is trivial and we omit it
for brevity. For $N=3$, at the end of Algorithm~\ref{alg:main},
$\theta^c_{\{1,2,3\}} = \{0,\,2\pi/3,\,4\pi/3\}$; $n_{\{1,2,3\}} = \{2,\,3,\,3\}$. Clearly, $C^l_2 = \pi/2 + \pi/3 - 2\pi/3 = \pi/6$ and $\theta^l_2 = \pi/3$. When agents $2$ and $1$ interact based on Algorithm~\ref{alg:repartition}, we get that $\theta_1^c = 0$ and $S_1 = 2\pi/3$. Clearly, the domain $Q$ is equipartitioned.

For the case $4 \leq N \leq 7$, if Algorithm~\ref{alg:repartition} terminates only after agents $2$ and $1$ have interacted, 
then it follows that all agents have learned about each other; their
areas of responsibility thus satisfy $S_i = 2\pi/N$ for all $i$ and 
that their areas of responsibility do not overlap. Thus, 
if Algorithm~\ref{alg:repartition} does not terminate before agents $2$ and $1$ have interacted, then it follows that $Q$ is 
equipartitioned.

At the end of Algorithm~\ref{alg:main}, the agents $N$ and $N-1$ only share a boundary and their partitions are minimally (lazily) sized at $2\pi/N$. Also, let $\alpha = \theta_{N-1}$; this will be a useful anchor for the subsequent calculations, as 
explained in Remark~\ref{rem:2}. At time $k = 1$, the agent $N-1$ interacts with agent $N-2$; this results in the agent $N-2$ learning about all $N$ agents and assigning itself an arc of responsibility
of size $2\pi/N$. Thus,
$$
\theta_{N-2}^c \leftarrow \alpha - \frac{2\pi}{N}
$$
The areas of responsibility of agents $N-1$ and $N-2$ only share a boundary located at $\alpha - \pi/N$. Suppose that this
process continues until time step $p \in \mathcal{T}$, $p\leq N-3$. Then, at the end of time step $p$, we get
$$
\theta^c_{N-p-1} \leftarrow \alpha - \frac{2\pi}{N}p,~S_{N-p-1} = \frac{2\pi}{N}
$$
At this point, it is worth noting that
$$
\theta^c_{N-p-2} = \begin{cases} \pi + \sum_{m=3}^{N-p-2}\frac{\pi}{m} - \frac{\pi}{N-p-1}, & p < N-3\\
0, & p = N-3
\end{cases}
$$
where we have prescribed that $\sum_{3}^2 (\cdot) = 0$ with slight abuse of notation.
If coverage is to be lost at this point, the following necessary and sufficient condition follows from Lemma~\ref{lem:sij}:
\begin{equation}
d(\theta^c_{N-p-1},\, \theta^c_{N-p-2}) > \frac{\pi}{N} + \frac{\pi}{N-p-1},~p \leq N-2
\label{eq:nsc1}
\end{equation}
It can be shown readily that Eq.~\eqref{eq:nsc1} is not 
satisfied for any permissible $p$ (i.e., $p \leq N-3$)
when $N \leq 6$. Thus, at $p = N-2$, agents $2$ and $1$
can interact as per Algorithm~\ref{alg:repartition} to yield
an equipartition of $Q$. On the other hand, when $N=7$,
\eqref{eq:nsc1} is satisfied for $p = 4$. Thus, there is
no overlap between the areas of responsibility of 
agents $1$ and $2$. It follows readily that there is an
instantaneous loss of coverage in $Q$ when $N=7$. This 
completes the proof.
\end{proof}

The loss of coverage for $N=7$ is shown in 
Fig.~\ref{fig:7 agent proof} which shows the status of coverage at each
step of Algorithms \ref{alg:main} and \ref{alg:repartition}. 

\begin{figure*}[htb]
\centering
        \includegraphics[width=.2\textwidth,trim={0cm 0cm 0cm 0cm}]{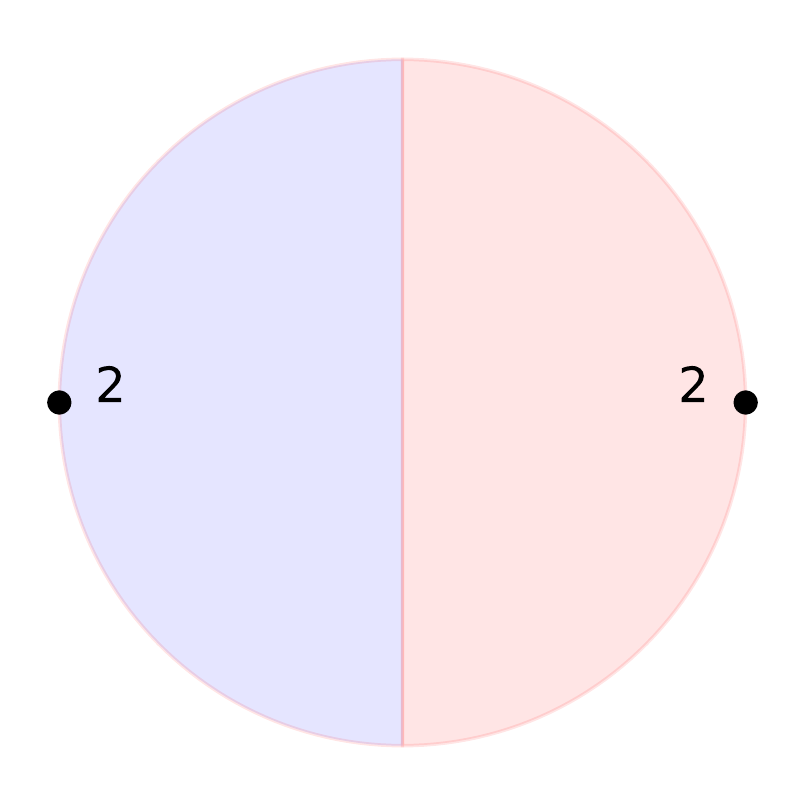}\hfill
        \includegraphics[width=.2\textwidth,trim={0cm 0cm 0cm 0cm}]{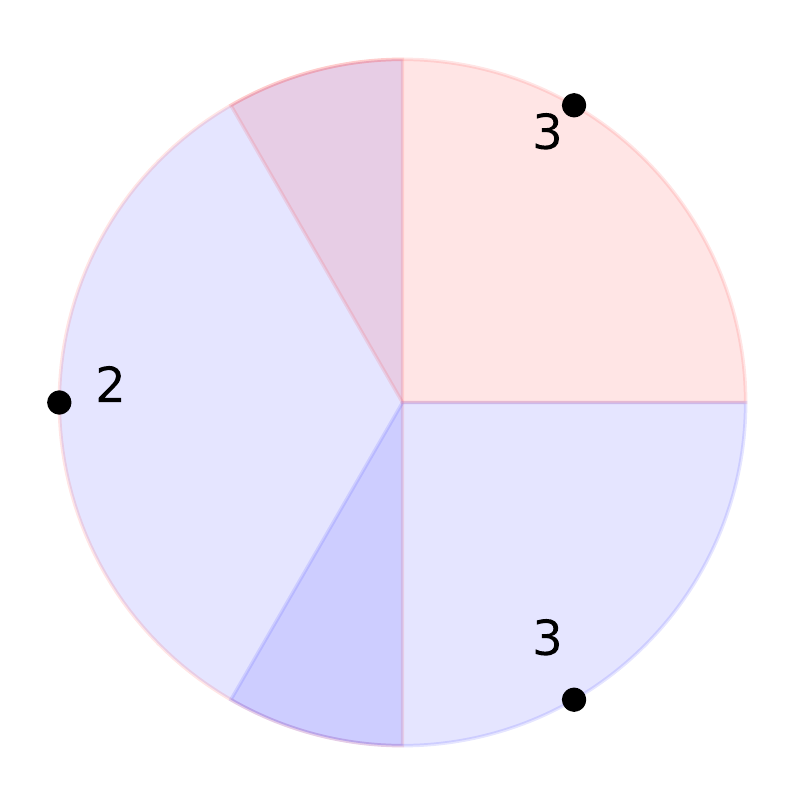}\hfill
        \includegraphics[width=.2\textwidth,trim={0cm 0cm 0cm 0cm}]{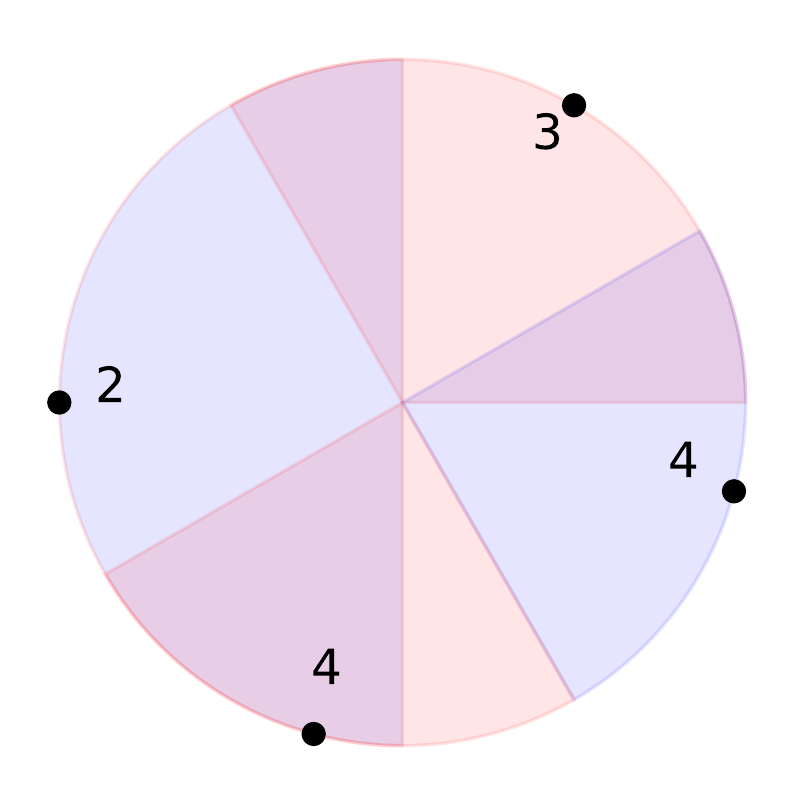}\hfill
        \includegraphics[width=.2\textwidth,trim={0cm 0cm 0cm 0cm}]{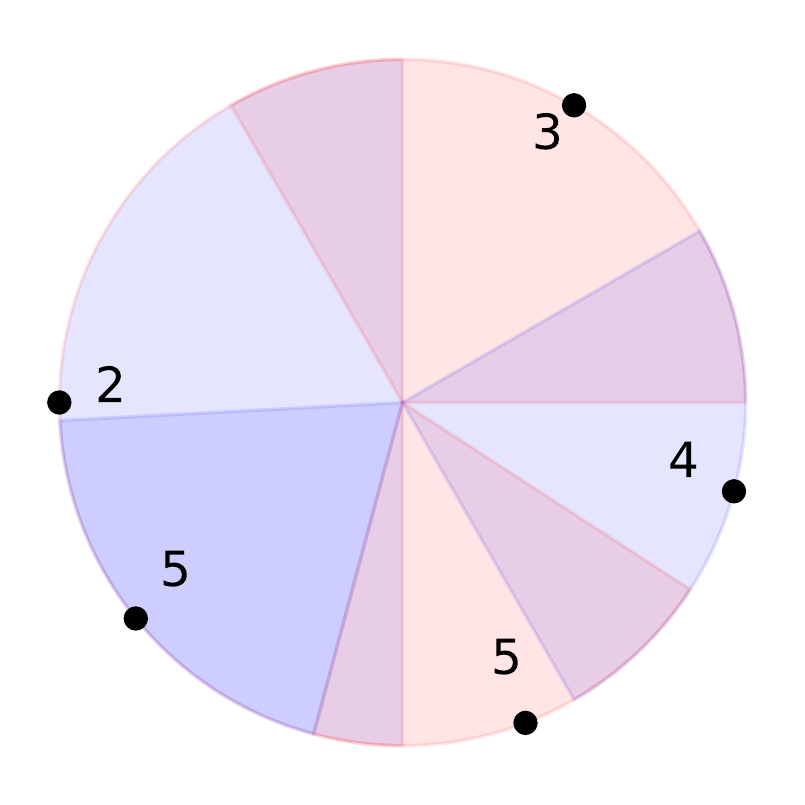}\hfill
        \includegraphics[width=.2\textwidth,trim={0cm 0cm 0cm 0cm}]{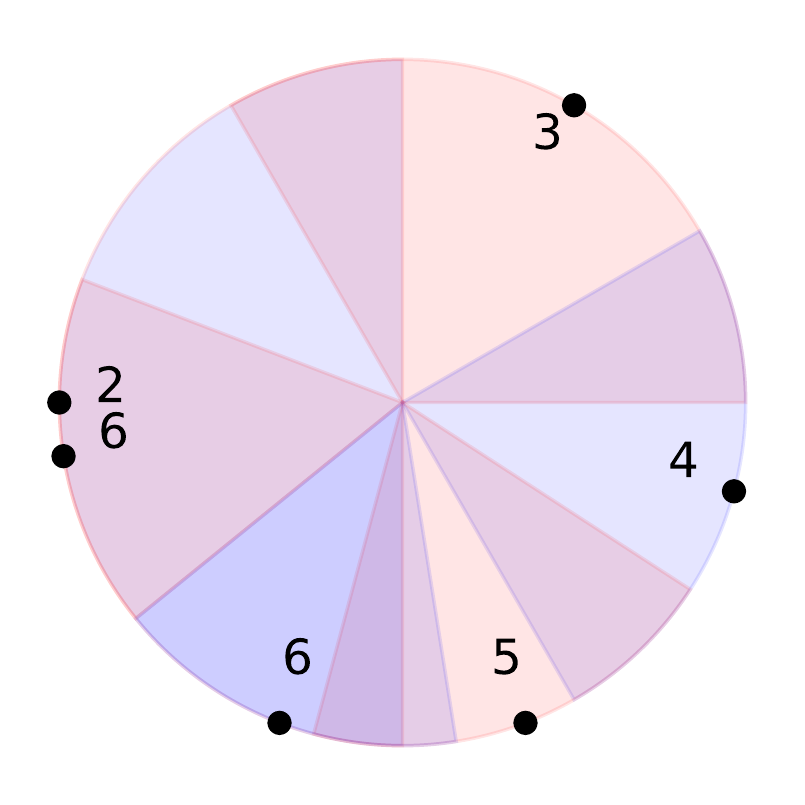} \hfill
        \includegraphics[width=.2\textwidth,trim={0cm 0cm 0cm 0cm}]{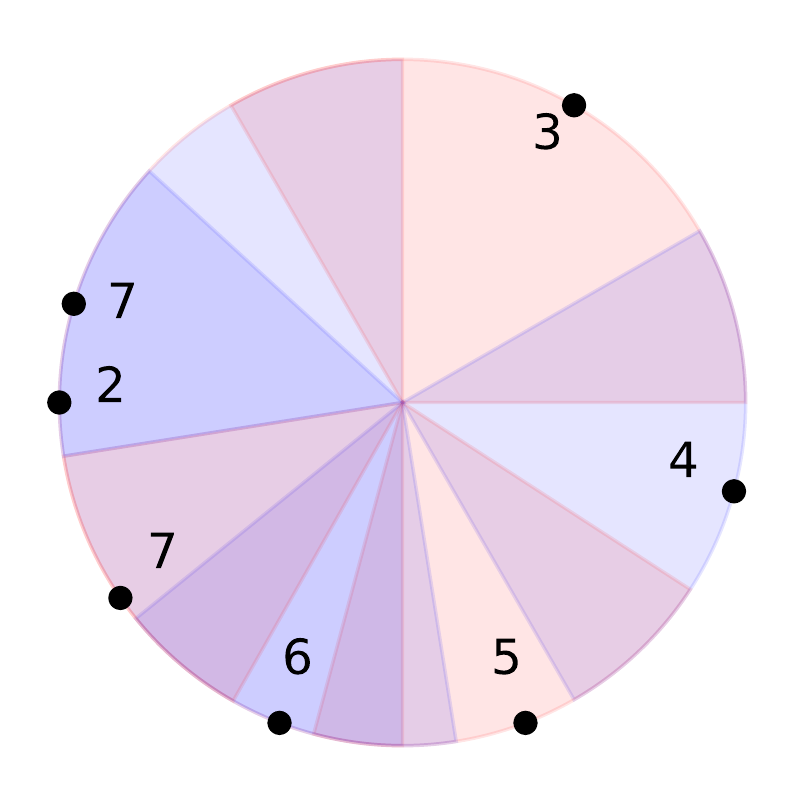}\hfill
        \includegraphics[width=.2\textwidth,trim={0cm 0cm 0cm 0cm}]{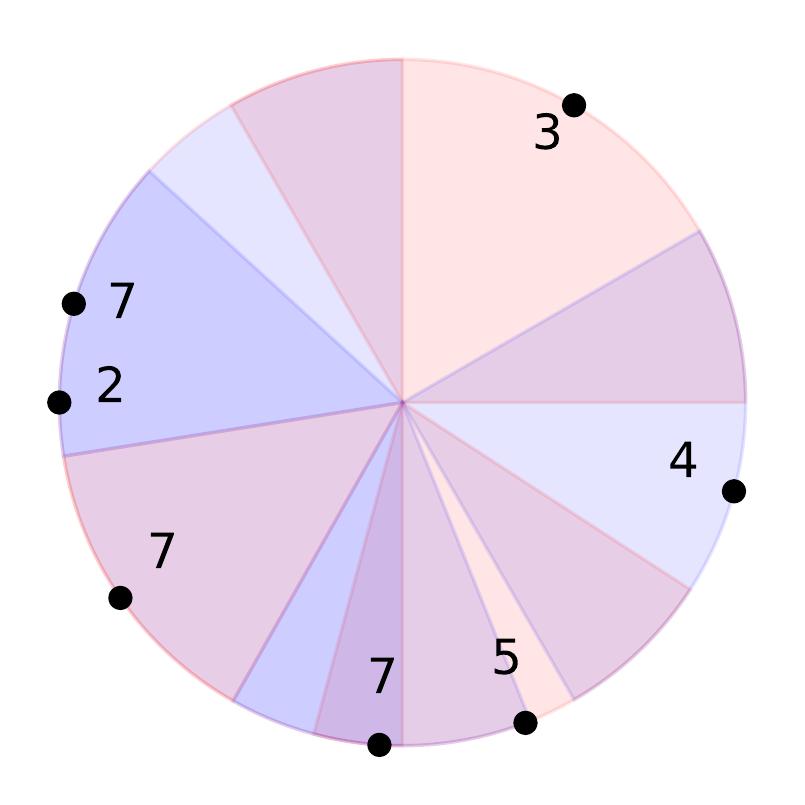}\hfill
        \includegraphics[width=.2\textwidth,trim={0cm 0cm 0cm 0cm}]{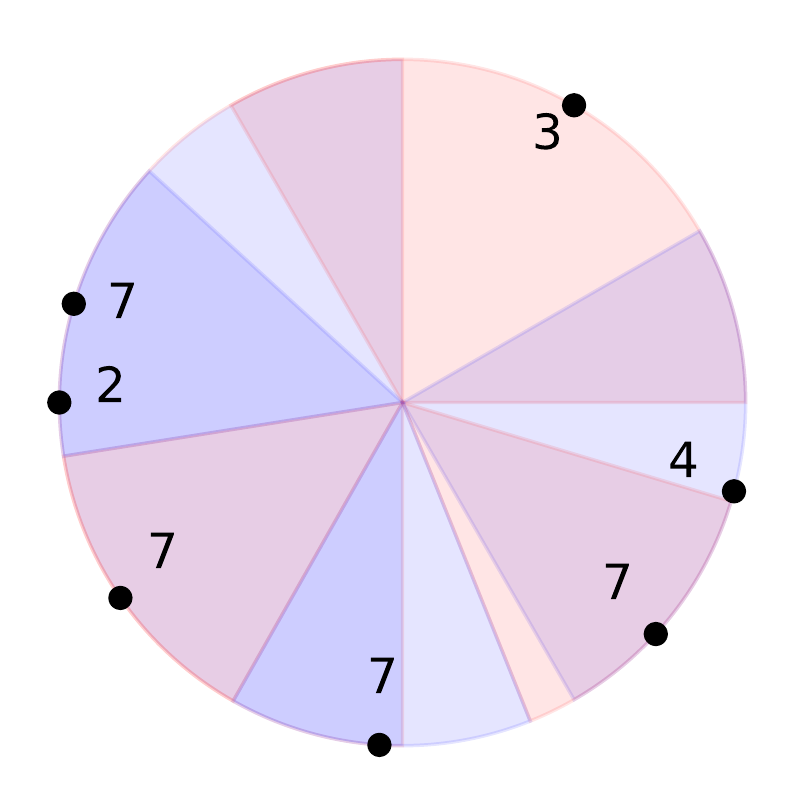}\hfill
        \includegraphics[width=.2\textwidth,trim={0cm 0cm 0cm 0cm}]{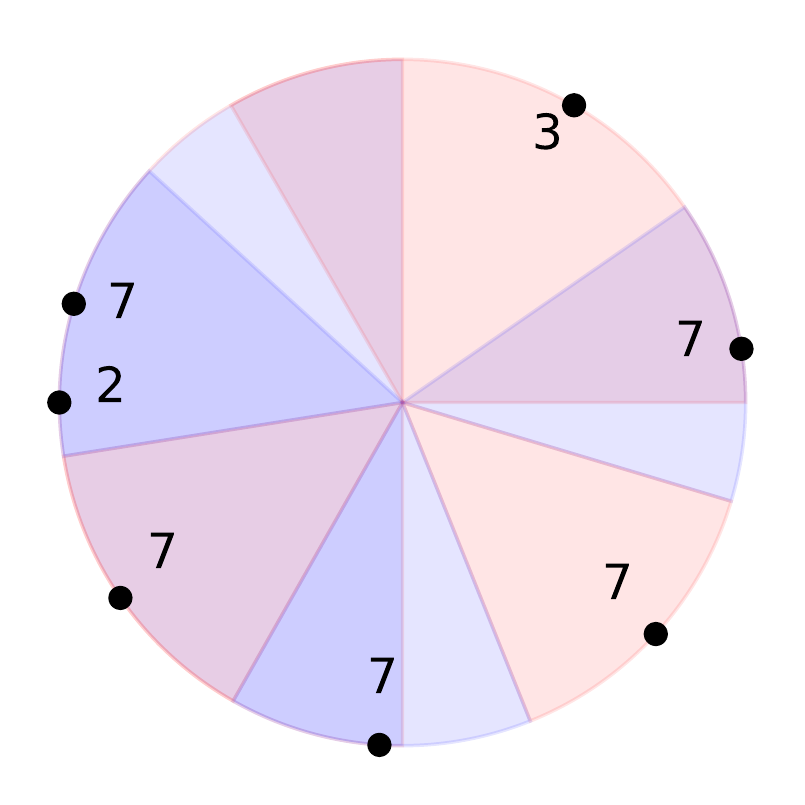}\hfill
        \includegraphics[width=.2\textwidth,trim={0cm 0cm 0cm 0cm}]{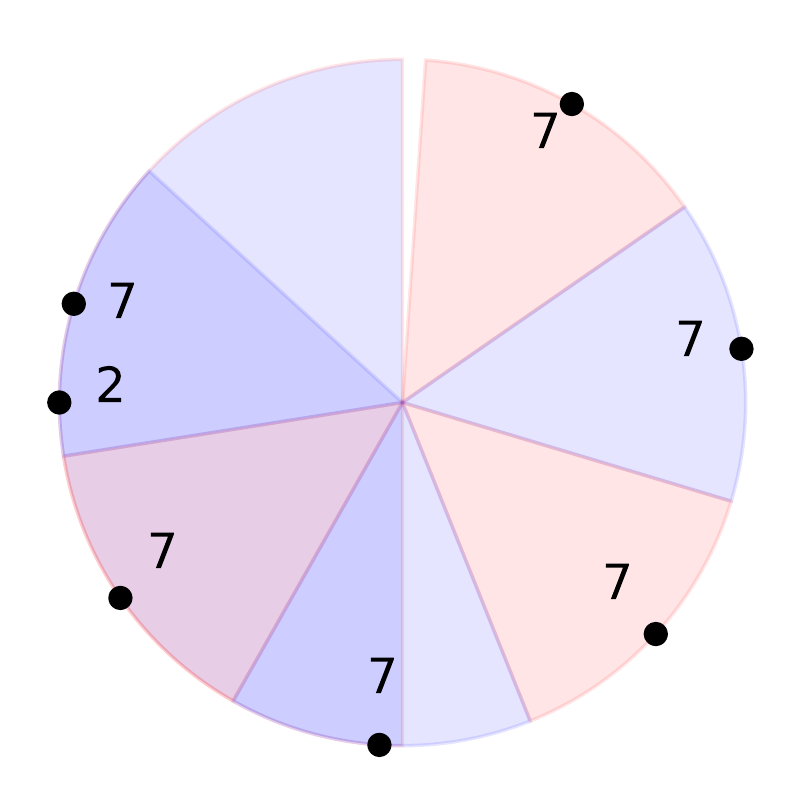}\hfill
    \centering
    \caption{Simulation of 7 agents added sequentially to $S^1$. The first $6$ images show the addition of agent $2$ - $7$. Images $7$ - $10$ show how repartitioning occurs in an anticlockwise fashion, and where failure occurs between two agents (the white region in the last plot of the second row).}
    \label{fig:7 agent proof}
\end{figure*} 

While Algorithm~\ref{alg:repartition} ensures that coverage is not
lost for $N \geq 6$, it is possible to find an alternate sequence
of events which lead to loss of coverage for $N = 5$. The sequence
of events is enumerated in Algorithm~\ref{alg:repartition1}. 
As with the previous sequence, the present sequence starts at
the of end of Algorithm~\ref{alg:main}. Thereafter, agent $N$
interacts with agent $1$. Formally, we define a sequence of events labelled by $\mathcal{T} = \{k\}$, $k \in \mathbb{N}$ and $k \leq N-2$, such that 
at instant $1$, the agent $N$ interacts with agent $1$ if $S_{N,1} > 0$; and for all subsequent $k$, agent $k-1$ interacts with agent $k$ if $S_{k-1,k} > 0$  or if $\theta^{u,l}_{k-1} = \theta^{l,u}_{k}$ (i.e., the borders overlap). Else, the sequence is halted. This sequence is enumerated formally in Algorithm~\ref{alg:repartition1}. 

\begin{algorithm}
\caption{Pairwise interaction algorithm - 2}
\label{alg:repartition1}
\begin{algorithmic}
\REQUIRE $N$ agents introduced as per Algorithm~\ref{alg:main}
\STATE $k = 1$, $run\_flag = 1$
\STATE Prescribe agent $k-1$ for $k = 1$ is agent $N$
\WHILE {$run\_{flag} = 1$ and $k \leq N-2$}
\IF {$S_{k-1,k} > 0$ or $\theta^{l,u}_{k} = \theta^{u,l}_{k-1}$}
\STATE Update knowledge of $k^{\rm th}$ agent
\STATE $\theta_{k}^c = \theta_{k-1}^c \oplus 2\pi/N$
\ELSE
\STATE $run\_{flag} = 0$; halt sequence
\STATE Check for coverage
\ENDIF
\STATE $k \leftarrow k + 1$
\ENDWHILE
\end{algorithmic}
\end{algorithm}

\begin{figure}[t]
\centering
\includegraphics[width=.3\textwidth,trim={0cm 0cm 0cm 0cm},clip]{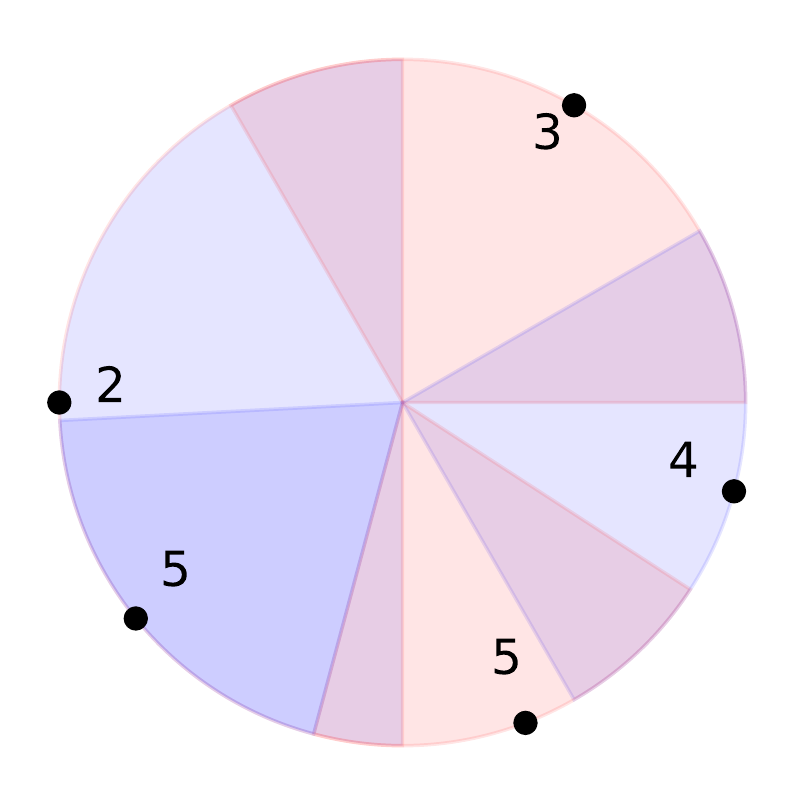}
\includegraphics[width=.3\textwidth,trim={0cm 0cm 0cm 0cm},clip]{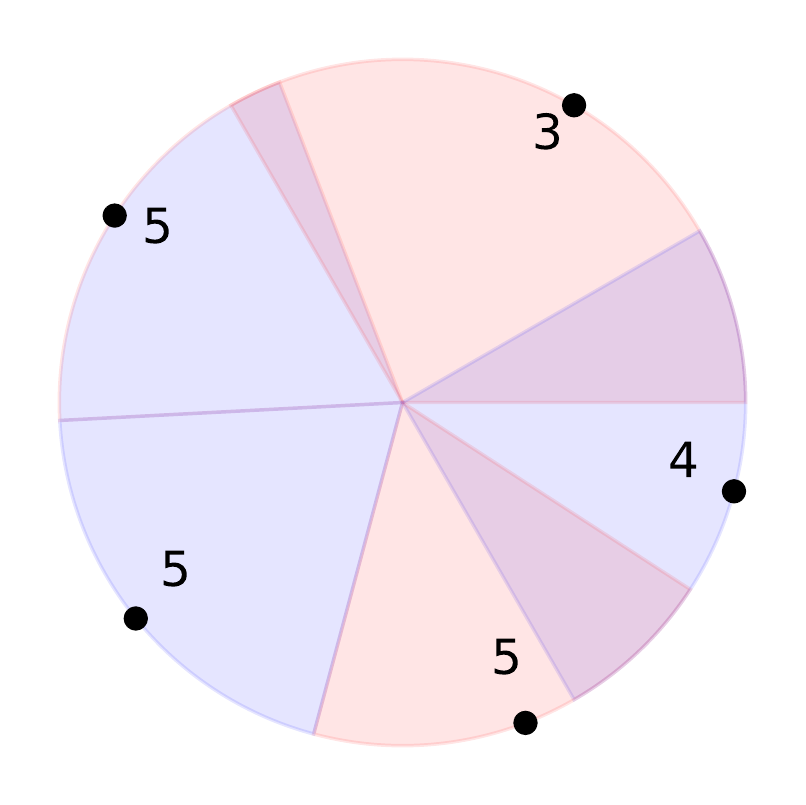}
\includegraphics[width=.3\textwidth,trim={0cm 0cm 0cm 0cm},clip]{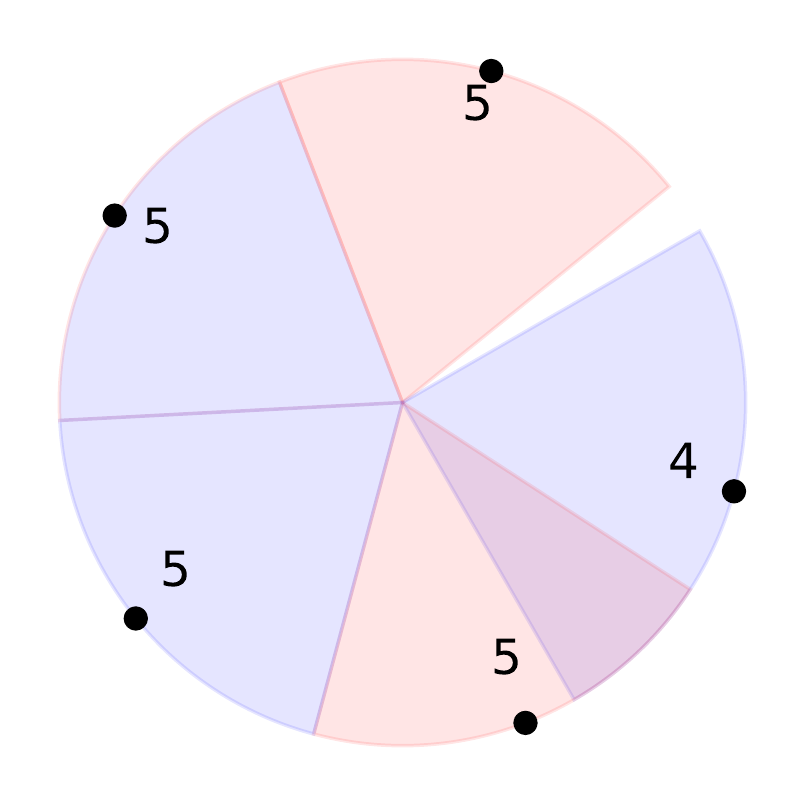}\hfill
\caption{Simulation of $5$ agents repartitioning following Algorithm~\ref{alg:repartition1}. After $2$ interactions, coverage is lost between the $2^{nd}$ and $3^{rd}$ agent. The size of the uncovered region $9^{\circ}$ or 
$0.157~{\rm rad}$.}
        \label{fig:5 agent fail}
\end{figure}

\begin{theorem}\label{thm:main2}
Let the number of agents be bounded by $N \leq 5$. 
Suppose that the agents are introduced following Algorithm~\ref{alg:main} and they then interact as per 
Algorithm~\ref{alg:repartition1}. Then, Algorithm~\ref{alg:repartition1} terminates with
loss of instantaneous coverage if and only if $N=5$. Moreover, for $N \leq 4$, 
Algorithm~\ref{alg:repartition} terminates with an equipartition of $Q$.
\end{theorem}
\begin{proof} The result follows trivially for 
$N = 1$ and $N = 2$ (where Algorithm~\ref{alg:repartition1}
is not necessary). The proof for the case $N=3$ is similar 
to that for the corresponding case in Thm~\ref{thm:main1},
except that agents $1$ and $3$ interact instead of 
$1$ and $2$. Thus, we need to address only cases $N=4$ and $N=5$. At the end of 
Algorithm~\ref{alg:main}, we note that
$$
\theta_N^c = \pi + \sum_{m=3}^N \frac{\pi}{m} = \begin{cases} 83\pi/60 & N = 4 \\ 107\pi/60 & N = 5\end{cases}
$$
for $N \leq 5$. Moreover, the overlap between agents $N$ and $1$ ($N \in \{4,\,5\}$) can be readily shown to be positive. Thus, 
at time $k = 1$ in $\mathcal{T}$, agent $1$ moves to
\begin{equation}
\theta^c_1 \leftarrow \alpha = \theta_N^c \oplus \frac{2\pi}{N} = \theta^N_c + \frac{2\pi}{N} - 2\pi
\end{equation}
We have introduced $\alpha$, as in the proof of Thm~\ref{thm:main1}, to serve as an anchor. In the subsequent interaction $p$, the agent $p$ interacts with agent $p-1$ assuming that Algorithm~\ref{alg:repartition1} has not terminated prematurely before that instant. If the algorithm first terminates prematurely at $p$, then a necessary and sufficient condition is that
\begin{eqnarray}
\nonumber & & \frac{\pi}{N} + \frac{\pi}{p+1} < \theta^c_{p} - \theta^c_{p-1} \\
\nonumber \Leftrightarrow & & \frac{\pi}{N} + \frac{\pi}{p+1} < \pi + \sum_{m=3}^p \frac{\pi}{m} - \frac{\pi}{p+1} - \theta^c_N  - \frac{2\pi (p-1)}{N} + 2\pi \\
\Leftrightarrow & & \frac{2\pi p}{N} + \frac{2\pi}{p+1}
+ \sum_{m=p+1}^N \frac{\pi}{m} < 2\pi + \frac{\pi}{N}
\label{eq:nsc2}
\end{eqnarray}
together with $p < N$. It can be checked readily that the condition is not satisfied
for $N=4$ and Algorithm~\ref{alg:repartition1} terminates with $Q$ being
equipartitioned. For $N = 5$, this condition is satisfied for $p = 3$; i.e., there is an instantaneous loss of coverage between agents $2$ and $3$. This 
completes the proof.
\end{proof}

\section{Generalization Using Numerical Experiments}\label{sec:sim}
The results presented in the previous section show how a pathological series
of interactions can lead to a loss of coverage when the number of agents is small,
but still larger than a critical threshold. The same machinery can be extended
to cases where the number of agents is larger, but closed-form solutions are not easy 
to calculate because of the geometric setting of the problem. However, the 
necessary and sufficient conditions in Eqs.~\eqref{eq:nsc1} and \eqref{eq:nsc2} can 
be examined through a numerical parametric study for larger values of $N$ than 
those considered in the previous section. We restrict this
study to Algorithm~\ref{alg:repartition} and note that the
analysis can be repeated readily for Algorithm~\ref{alg:repartition1}.

It can be checked readily that the separation between two neighboring agents 
$k$ and $k+1$ ($2 < k < N-1$) at the end of Algorithm~\ref{alg:main} is given by
$$
\theta^c_{k+1} - \theta_k^c = \frac{(k+3)\pi}{(k+1)(k+2)},~~
\theta^c_N - \theta_{N-1}^c = \frac{2\pi}{N}
$$ 
Notice that, for large $k$, $\theta^c_{k+1} - \theta_k^c \approx \pi/(k+1)$. When $N$ become large, there exists $p$ such that the application of 
Algorithm~\ref{alg:repartition} and the accompanying interaction between agents $p+1$ and $p$ causes $\theta_p^c \prec \theta_{p-1}^c$ (informally, agent $p$ ``crosses'' $p-1$) . We refer to this as the C-crossover. There also exists $q$ such 
that $\theta_q^u \prec \theta_{q-1}^l$ and $S_{q,q-1} = 0$. We refer to this as the UL-crossover. 
The crossover index (the agent which crosses over its predecessor) 
is shown in Fig.~\ref{fig:crossover}, with a crossover value of $0$ indicating no crossovers. Note that neither of
these crossovers corresponds to loss of coverage; the UL-crossover means, in particular, that 
Algorithm~\ref{alg:repartition} cannot be applied in its present form once the crossover happens. Moreover,
the applicability of the analytical machinery developed in the proof of Thm~\ref{thm:main1} is restricted  
to $N \leq 19$.

The application of Algorithm~\ref{alg:repartition1} numerically for $N \in [8, 19]$ (the case $N \leq 7$ is covered
using Thm~\ref{thm:main1}) shows that Algorithm~\ref{alg:repartition1} terminates prematurely with loss
of continuous coverage as follows: between agents $1$ and $2$ for $N \in [7, 11]$; between agents $2$ and $3$
for $N \in [12, 16]$; and between agents $3$ and $4$ for $N \in [17, 19]$.

\begin{figure}[ht]
\centering
\includegraphics[width=0.6\textwidth]{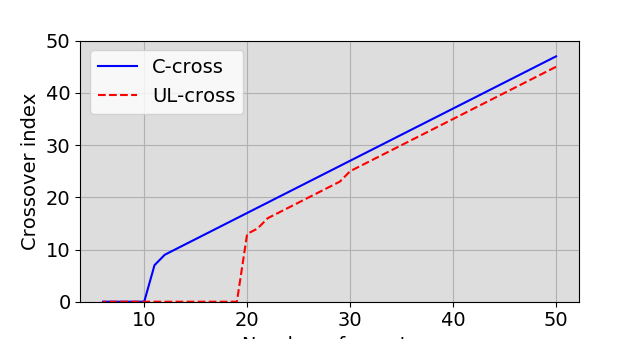}
\caption{The index of the first agent (starting with $N$) to cross its predecessor under Algorithm~\ref{alg:repartition}.
An index of $0$ implies that no crossover takes place.}
\label{fig:crossover}
\end{figure}

Note that a UL-crossover cannot be avoided when events occur as per Algorithm~\ref{alg:repartition} for large $N$. We investigate
a naive extension whose pseudocode is presented in Algorithm~\ref{alg:repartition2}. Notice that Algorithm~\ref{alg:repartition2}
involves a reversion to interaction with the nearest counter-clockwise neighbor and it halts when there is no
overlap between an agent and its nearest counter-clockwise
neighbor. From Fig.~\ref{fig:uncovered} it is evident
that the uncovered area reduces with increasing $N$, although the trend is not monotonic.
Although the size of the uncovered area reduces
rapidly with increasing $N$, the partition
at the end of Algorithm~\ref{alg:repartition2}
is seen to not be an equipartition.

\begin{algorithm}
\caption{Naive extension of Algorithm~\ref{alg:repartition}}
\label{alg:repartition2}
\begin{algorithmic}
\REQUIRE $N$ agents introduced as per Algorithm~\ref{alg:main}
\REQUIRE Algorithm~\ref{alg:repartition} run until
premature termination at agent $p$
\STATE $k = 1$, $run\_flag = 1$, $p = N$
\WHILE {$run\_{flag} = 1$ and $k \leq k_{\rm max}$}
\STATE Solve $p_c = \arg\min_{j} (d(\theta^c_p, \theta^c_j)~|~\theta_c^j \prec \theta^c_p)$
\IF {$S_{p,j} > 0$ or $\theta^l_{p} = \theta^u_{j}$}
\STATE Update: $\theta_{p_{c}}^c 
\leftarrow \theta_{p}^c \ominus 2\pi/N$; $p \leftarrow p_c$
\ELSE
\STATE $run\_{flag} = 0$; halt sequence
\ENDIF
\STATE Check for coverage
\IF {Q is equipartitioned}
\STATE $run\_{flag} = 0$; halt sequence
\ENDIF
\STATE $k \leftarrow k + 1$
\ENDWHILE
\end{algorithmic}
\end{algorithm}

\begin{figure}[ht]
\centering
\includegraphics[width=0.6\textwidth]{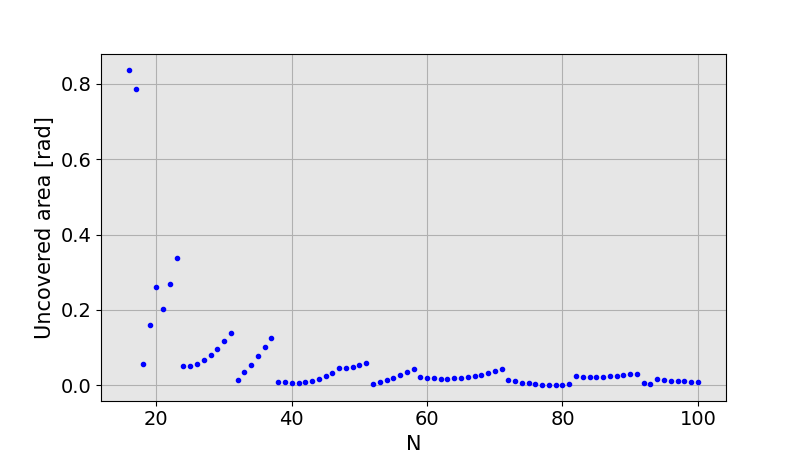} 
\caption{The uncovered area after Algorithm~\ref{alg:repartition2} terminates.}
\label{fig:uncovered}
\end{figure}

\section{Concluding Discussion}\label{sec:conc}
In this paper, we set out to investigate how the desirable properties of coverage algorithms may change with the number of agents, when the inter-agent communication is discrete, local, and event-driven. We modified a
well-known coverage algorithm by prescribing that agents use a certain 
lazy logic to repartition and resize their areas of responsibility. We applied
this algorithm to a simple problem involving lazy agents introduced sequentially
into a 2D annular domain. The annular geometry permitted our analysis to be restricted to a unit circle. We constructed a sequence of events that yields an
equipartitioned domain for a small but nontrivial number of agents, but fails when the number of agents exceeds a certain threshold. We
conducted numerical experiments to demonstrate how the algorithm performs
when the number of agents becomes large enough to the point where theoretical
analysis is no longer feasible using our methods.

Although carried out in a simplified setting, our work illustrates how the performance guarantees of coverage algorithms can be sensitive to the number of agents, unless the performance guarantees are proven rigorously beforehand 
for an arbitrary number of agents
(which is difficult for general problems involving event-triggered, gossip-based communication). It is not a straight-forward scaling problem, and 
the actual number of agents plays a critical role in the nature of the guarantees.

We assumed that the agents repartition and resize their areas of responsibility lazily (i.e., with $\epsilon =  0$ in Algorithm~\ref{alg:general}). 
Our results show that a degree of altruism might be necessary in order to guarantee coverage using a manifestation of Algorithm~\ref{alg:general} that works for an arbitrary number of agents. It remains an open problem to determine if there exists a sequence $\epsilon[t]$ which guarantees that Algorithm~\ref{alg:general}) leads to an equipartitioned domain for any sufficiently rich sequence of events.

\bibliography{BibMain}
\bibliographystyle{plain}

\end{document}